\documentclass[12pt]{article}

\usepackage{fullpage}
\usepackage{amsthm}
\usepackage{amsmath}
\allowdisplaybreaks
\usepackage{graphicx} 
\usepackage{wrapfig} 
\usepackage{ifthen}
\usepackage{xifthen}
\usepackage{authblk}

\newtheorem{theorem}{Theorem} 
\newtheorem{lemma}{Lemma} 
 
\newtheorem{algorithm}{Algorithm}

\newcommand{\q}[1]{
  \ifthenelse{\isempty{#1}}{
    {Q}
  }{%else
    {Q({#1})}
  }
}
\newcommand{\qa}[1]{
  \ifthenelse{\isempty{#1}}{
    \tilde{Q}
  }{%else
    \tilde{Q}(#1)
  }
}
\newcommand{\p}[1]{
  \ifthenelse{\isempty{#1}}{
    P
  }{%else
    P(#1)
  }
}
\newcommand{\dist}[2]{ d({#1},{#2}) } 

\newcommand{\s}{{\sigma}}
\newcommand{\lam}{{\lambda}}
\newcommand{\D}[1]{D({#1})}
\newcommand{\Dt}[1]{D'({#1})}
\newcommand{\greedy}{{\sc GreedyTrack}}
\newcommand{\discGreedy}{{\sc ModifiedGreedy}}
\newcommand{\feasability}{Gap Invariant}

\newcommand{\spFeasability}{SP-Gap Invariant}
\newcommand{\path}[1]{\Pi_{#1}}
\long\def\cut#1{{}}

\title{Trackability with Imprecise Localization\thanks{%
    This research was supported in part by the National Science
    Foundation grants IIS-0904501, CNS-1035917, and the National
    Science Foundation Graduate Research Fellowship under Grant
    No. DGE-1144085.}}  \author{Kyle Klein} \author{Subhash Suri}
\affil{Department of Computer Science\\
  University of California Santa Barbara\\
  Santa Barbara, CA, USA 93106\\
	\{kyleklein, suri\}@cs.ucsb.edu}

\date{}

\usepackage{subfigure}

\begin{document}

\maketitle

\begin{abstract}

  Imagine a tracking agent $\p{}$ who wants to follow a moving target $\q{}$ 
  in $d$-dimensional Euclidean space. The tracker has access to a noisy
  location sensor that reports an estimate $\qa{t}$ of the target's
  true location $\q{t}$ at time $t$, where $||\q{t} - \qa{t}||$
  represents the sensor's localization error. We study the limits of
  tracking performance under this kind of sensing imprecision. In
  particular, we investigate (1) what is $\p{}$'s best strategy to
  follow $\q{}$ if both $\p{}$ and $\q{}$ can move with equal speed,
  (2) at what rate does the distance $||\q{t} - \p{t} ||$ grow under
  worst-case localization noise, (3) if $\p{}$ wants to keep $\q{}$
  within a prescribed distance $L$, how much faster does it need to
  move, and (4) what is the effect of obstacles on the tracking
  performance, etc.  Under a relative error model of noise, we are
  able to give upper and lower bounds for the worst-case tracking
  performance, both with or without obstacles.
%% in $d$-dimensional Euclidean spaces.

\end{abstract}

\section{Introduction}

The problem of tracking a single known target is a classical one with a long 
history in artificial intelligence, robotics, computational geometry, graph 
theory and control systems. The underlying motivation is that many robotic
applications including search-and-rescue, surveillance, reconnaissance
and environmental monitoring have components that are best modeled as
a tracking problem. The problem is often formulated as a
\emph{pursuit-evasion} game, with colorful names such as
Man-and-the-Lion, Cops-and-Robbers, Hunter-and-Rabbit, Homicidal
Chauffeur, and
Princess-and-Monster~\cite{aigner,alpern-princess,bopardikar-chauffeur,isler-cop-robber}. Visibility-based
pursuit evasion~\cite{Guibas99,suzuki-visibility}, in particular, has
been a topic of great interest, in part due to its simple but realistic model: 
a team of pursuers is tasked with locating a single
adversarial evader in an geometric environment with polygonal
obstacles where pursuers learn the evader's position only when the
latter is in their line-of-sight. After two decades of research, tight
bounds are known for detection or capture of the evader for many basic
formulations of the problem~\cite{Bhadauria2012,Guibas99,KleinSocg13},
although the topic remains a rich subject of ongoing
research~\cite{kleinIsaac13,Noori2013}.

Most theoretical analyses of tracking, however, assume an idealized
sensing model, ignoring the fact that all location sensing is noisy
and imprecise in practice: the target's position is rarely known with
complete and error-free precision. Although some papers have explored
models to incorporate practical limitations of idealized visibility
including angular visibility~\cite{karnad08}, beam
sensing~\cite{park}, field-of-view sensors~\cite{gerkey-fov}, and
range-bounded visibility~\cite{Bopardikar2008}, the topic of sensing
noise or imprecision has largely been handled heuristically or through
probabilistic techniques such as Kalman filters~\cite{kalman1,Liao03,Sheng2005,thrun1}.  
One exception is~\cite{Rote2003}, where Rote investigates a tracking problem 
under the \emph{absolute error} model: in this model, the target's position 
is always known to lie within distance $1$ of its true location, regardless 
of its distance from the tracker. The analysis in~\cite{Rote2003} shows that, 
under this noise model, the distance between the tracker and the target can 
grow at the rate of $\Theta(t^{1/3})$, where $t$ is the time parameter.  Our
model, by comparison, deals with a more severe form of noise, with imprecision 
proportional to the distance from the tracker. 
In~\cite{Kuntsevich2002}, Kuntsevich et al. consider the same relative 
error model as ours, but without any obstacles. Their work has a 
control-theoretic perspective, with a primary goal of deriving a bound on 
the time needed by the tracker to capture the target.
Our main contribution is to analyze the worst-case behavior of 
trackability as a function of the localization precision parameter $\lam$.

\paragraph{Motivation and the Problem Statement.}

This paper takes a small step towards bridging the gap between theory
and practice of trackability, and analyzes the effect of noisy sensing. 
In particular, we consider a tracking agent $\p{}$ who wants to follow a 
moving target $\q{}$ in $d$-dimensional Euclidean space using a
noisy location sensor. 
For simplicity, we analyze the problem in two dimensions, but the results
easily extend to $d$ dimensions, as discussed in Section~\ref{sec:dimension}. 
We use the notation $\q{t}$ and $\p{t}$
to denote the (true) positions of the target and the tracker at time
$t$. We adopt a simple but realistic model of \emph{relative} error in
sensing noise: the localization error is \emph{proportional} to the
true distance between the tracker and the target.  More precisely, the
localization error is upper bounded as $||\q{t} - \qa{t}|| \:\leq\:
\frac{1}{\lam} || \p{t} - \q{t} ||$ at all times $t$, where $\lam \geq
1$ is the quality measure of \emph{localization precision}. Thus, the
closer the target, smaller the error, and a larger $\lam$ means better
localization accuracy, while $\lam = 1$ represents the completely
noisy case when the target can be anywhere within a disk of radius
$||\p{t} - \q{t} ||$ around $\q{t}$.
\emph{It is important to note that the parameter $\lam$ is used only for the 
analysis, and is not part of information revealed to the tracker.} In other words, 
the tracker only observes the approximate location $\qa{t}$, and \emph{not the
uncertainty disk} containing the target.
The relative error model is intuitively simple (farther the object, larger the 
measurement error) and captures the realism of many sensors: for instance, the 
resolution error in camera-based tracking systems is proportional to the target's 
distance, and in network-based tracking, \emph{latency} causes a proportionate 
localization uncertainty because of target's movement before the signal is received 
by the tracker.

We study the tracking problem as a game between two players, the tracker $\p{}$ and 
the target $\q{}$, which is played in \emph{continuous time and space}: that is, 
each player is able to instantaneously observe and react to other's position, and
the environment is the two-dimensional plane, with or without polygonal obstacles. 
Both the target and the tracker can move with equal speed, which we normalize to 
\emph{one}, without loss of generality. With the unit-speed assumption,
the following holds, for all times $t_1 \leq t_2$:

\begin{equation*}
||\q{t_2} - \q{t_1}|| \leq |t_2 - t_1|, \ \ \ ||\p{t_2} - \p{t_1}|| \leq |t_2 - t_1|
\end{equation*}

Under the \emph{relative localization error} model, the reported location of the target
$\qa{t}$ always satisfies the following bound, where $\lam$ is the accuracy parameter:

\begin{equation*}
  ||\q{t} - \qa{t} || \:\leq\: \frac{||\p{t} - \q{t}||}{\lam}
\end{equation*}

We measure the tracking performance by analyzing the distance function between 
the target and the tracker, namely, $\D{t} = \dist{\p{t}}{\q{t}}$, over time,
with $\D{0}$ being the distance at the beginning of the game.
Under error-free localization, the distance remains bounded as $\D{t} \leq \D{0}$.
We analyze how $|| \D{t} - \D{0}||$ grows under the relative error model,
as a function of $\lam$.
Our main results are as follows.

\paragraph{Our Results.}

We show that the simple greedy strategy of ``always move to the
\emph{observed} location of the target'' achieves $\D{t} \;\leq\;
\D{0} + t/\lam^{2}$.  That is, the target's distance from the tracker
can grow at most at the rate of $O(\lam^{-2})$, the inverse quadratic
function of the localization parameter. We prove this rate to be
worst-case optimal with a matching lower bound: a strategy for the
target that ensures that, under the relative error model, it can
increase its distance as $\D{t} \: \geq \: \D{0} + \Omega(t
/\lam^{2})$.

We then extend this analysis to environments with polygonal obstacles,
and show that the tracker can increase its distance by $\Omega(t)$ in
time $t$ \emph{for any finite $\lam$}. This is unsurprising because
two points within a small margin of sensing error can be far apart in
free-space, thereby fooling the tracker into ``blind alleys.''  More
surprisingly, however, if we adopt a localization error that is
proportional to the \emph{geodesic} distance (and not the Euclidean
distance) between the target and the tracker, then the distance
increases at a rate of $\Theta (\lam^{-1})$. This bound is also tight
within a constant factor: the tracker can maintain a distance of
$\D{t} \;\leq\; \D{0} + O(t/\lam)$ by the greedy strategy, while the
target has a strategy to ensure that the distance function grows as at
least $\D{t} \;\geq\; \D{0} + \Omega(t/\lam)$.

Our analysis also helps answer some other questions related to
tracking performance.  For instance, a natural way to achieve good
tracking performance in the presence of noisy sensing is to let the
tracker move at a faster speed than the target.  Then, what is the
minimum speedup necessary for the tracker to reach the target (or,
keep within a certain distance of it)?  We derive upper and lower
bounds for this speedup function, which are within a constant
factor of each other as long as $\lam \geq 2$.

\paragraph{Organization.}

The paper is organized as follows. Section~\ref{sec:unobstructed}
investigates the tracking problem in the Euclidean plane without
obstacles.  Section~\ref{sec:speed} derives bounds on the speedup
function.  Section~\ref{sec:obstructed} explores the tracking
performance in the presence of polygonal
obstacles. Section~\ref{sec:dimension} briefly discussed the
straight-forward extension of our results to $d$-dimensions.
Section~\ref{sec:conclusion} offers a summary and concluding
remarks.

\section{Tracking in the Unobstructed Plane}
\label{sec:unobstructed}

We begin with the simple setting in which a tracking agent $\p{}$ wants to follow 
a moving target $\q{}$ in the two-dimensional plane without any obstacles.
We show that the trivial ``aim for the target's observed location'' achieves 
essentially the best possible worst-case performance. We first prove the upper 
bound on the derivative $\Dt{t}$ of the distance function $\D{t}$, and then 
describe an adversary's strategy that matches this upper bound.

\subsection{Tracker's Strategy and the Upper Bound}

Our tracker uses the following obvious algorithm, whose performance is analyzed
in Theorem~\ref{thm:unobstructed-upper} below.

\paragraph{\greedy .}
At time $t$, the tracker $\p{}$ moves directly towards the target's observed location $\qa{t}$.

\begin{theorem}
By using {\greedy}, the tracker can ensure that $\D{t} \leq \D{0} + O(t / \lam^{2})$, for all $t \geq 0$.
  \label{thm:unobstructed-upper}
\end{theorem}

\begin{proof} 
Consider the true and the observed positions of the target, namely  $\q{t}$ and $\qa{t}$,
respectively, at time $t$, and let $\gamma$ be the angle formed by them at $\p{t}$. 
See Figure~\ref{fig:unobstructed-upper}. Consider an arbitrarily small time period 
$\Delta t$ during which $\p{}$ moves towards $\qa{t}$ and $\q{t}$ moves away from 
$\p{t}$. We want to compute the derivative of the distance function, given as
Equation~(\ref{eqn:upperLimit}).

  \begin{equation}
    \label{eqn:upperLimit}
    \Dt{t} \:=\: \lim_{\Delta t \rightarrow 0}\frac{\D{t + \Delta t} - \D{t}}{\Delta t}
  \end{equation}

\begin{figure}[htbp]
\begin{center}
\includegraphics[height=5cm]{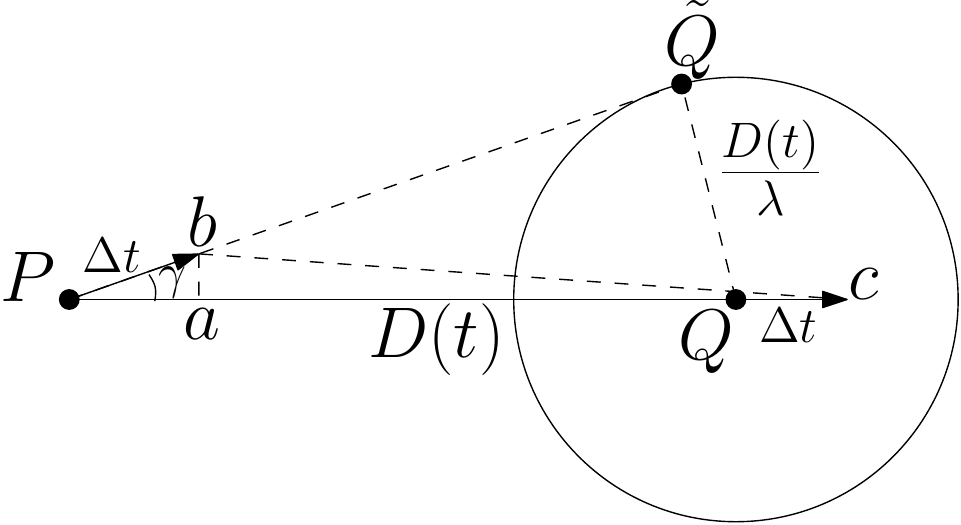}
\caption{Proof of Theorem~\ref{thm:unobstructed-upper}.}
\label{fig:unobstructed-upper}
\end{center}
\end{figure}

The new distance between the target and the tracker is given by $bc$
in Fig.~\ref{fig:unobstructed-upper}.  In the triangle $abc$, we have
$ab = \Delta t\sin\gamma$ and $ac = \D{t} + \Delta t - \Delta
t\cos\gamma$.  We, therefore, can bound $\D{t+\Delta t}$ as follows
(where the final inequality uses the fact $\sqrt{1+x}\leq 1 + \frac{x}{2})$:
    \begin{align*}
      \D{t + \Delta t}
        & = \sqrt{(\Delta t\sin \gamma)^2 + (\D{t} + \Delta t - \Delta t\cos\gamma)^2}\\
 	& = \sqrt{\Delta t^2\sin^2 \gamma + \D{t}^2 + 2\D{t}\Delta t(1 - \cos \gamma) + 
	\Delta t^2-2\Delta t^2\cos \gamma + \Delta t^2\cos^2 \gamma}\\
        & =  \sqrt{\Delta t^2 + \D{t}^2 + 2\D{t}\Delta t(1 - \cos \gamma) +\Delta t^2 
	-2\Delta t^2\cos \gamma}\\
        & = \sqrt{(\D{t}+\Delta t)^2 - 2\D{t}\Delta \cos \gamma  +\Delta t^2 -2\Delta 
		t^2\cos \gamma}\\
        & = \sqrt{(\D{t}+\Delta t)^2 -2\Delta t(\D{t} + \Delta t) +\Delta t^2 - 
		2\D{t}\Delta \cos \gamma -2\Delta t^2\cos \gamma + 2\Delta t(\D{t} + \Delta t)}\\
        & = \sqrt{(\D{t}+\Delta t - \Delta t)^2 - 2\D{t}\Delta t\cos \gamma  
		-2\Delta t^2\cos \gamma + 2\Delta t(\D{t} + \Delta t)}\\
        %& = \sqrt{\D{t}^2 - 2\D{t}\Delta t\cos \gamma  -2\Delta t^2\cos \gamma + 2\Delta t(\D{t} + \Delta t)}\\
        & = \sqrt{\D{t}^2 + 2\Delta t(\D{t} + \Delta t)(1 - \cos \gamma)}\\
        & = \D{t}\sqrt{1 + 2\Delta t(\D{t} + \Delta t)(1 - \cos \gamma)/\D{t}^2}\\
        & \leq \D{t} + (\Delta t)(1 + \Delta t/\D{t})(1 - \cos \gamma)
    \end{align*}

Returning to Equation~(\ref{eqn:upperLimit}), we get

\cut{%-----------
 \begin{equation*}
    \begin{split}
      \Dt{t} & = \lim_{\Delta t \rightarrow 0} \frac{\D{t + \Delta t} - \D{t}}{\Delta t}\\
            & \leq \lim_{\Delta t \rightarrow 0} (1 + \Delta t /\D{t})(1- \cos \gamma)\\
            & = 1 - \cos \gamma
    \end{split}
 \end{equation*}
}%-------------

\[
\Dt{t} \:=\: \lim_{\Delta t \rightarrow 0} \frac{\D{t + \Delta t} - \D{t}}{\Delta t}
            \:\leq\: \lim_{\Delta t \rightarrow 0} (1 + \Delta t /\D{t})(1- \cos \gamma)
            \:\:=\:\: 1 - \cos \gamma
\]

Finally, since $\sin \gamma \leq \frac{1}{\lam}$, we get $\Dt{t} \leq 1 - \sqrt{1-\frac{1}{\lam^2}}$,
which simplifies by the Taylor series expansion:

\cut{%-------
\begin{equation*}
  \begin{split}
    \Dt{t} & \leq 1 - (1 - \frac{1}{2\lam^2} - \frac{1}{8\lam^4} - \cdots)\\
          & = \frac{1}{2\lam^2} + \frac{1}{8\lam^4} + \cdots \leq \frac{1}{\lam^2}
  \end{split}
\end{equation*}
}%

\[
\Dt{t} \: \leq \:1 - (1 - \frac{1}{2\lam^2} - \frac{1}{8\lam^4} - \cdots)
      \: = \:\frac{1}{2\lam^2} + \frac{1}{8\lam^4} + \cdots 
     \:\:\leq\:\: \frac{1}{\lam^2}
\]

This completes the proof that $\D{t} \:\leq\: \D{0} + t / \lam^2$.
\end{proof}

\subsection{Target's Strategy and the Lower Bound}

We now show that this bound is asymptotically tight, by demonstrating a strategy for 
the target to grow its distance from the tracker at the rate of 
$\D{t} \:\geq\: \D{0} + \Omega ( t / \lam^2 )$, for all $t \geq 0$.
We think of the target as an adversary who can choose
its observed location at any time subject only to the constraints of
the error bound: $||\q{t} - \qa{t}|| \leq \frac{1}{\lam}(||\p{t} - \q{t}||)$.
(Recall that the tracker only observes the location $\qa{t}$, and has
no direct knowledge of either the parameter $\lam$ or the distance
$||\p{t} - \q{t}||$.  Those quantities are only used in the
analysis. However, the lower bound holds even if the tracker knows the uncertainty 
disk, namely, the localization error $\frac{1}{\lam}(||\p{t} - \q{t}||)$.)

In order to analyze the lower bound, we divide the time into
\emph{phases}, and show that the distance from the tracker increases
by a \emph{multiplicative factor} in each phase, resulting in a growth
rate of $\Omega(1 + \lam^{-2})$. If the $i$th phase begins at time $t_i$,
then we let $d_i=||\q{t_i} - \p{t_i}||$ denote the distance between
the target and the tracker at $t_i$. During the $i$th phase, the
target maintains the following invariant for a constant $0 < \alpha <
1$ to be chosen later.

\paragraph{\feasability .} \label{inv:alphaInvar} 
Throughout the $i$th phase, the target moves along a path $\q{t}$ such that
$||\q{t} - \p{t}|| \geq \alpha d_i$, for all times $t$, and all reported locations 
satisfy $||\q{t} - \qa{t}|| \leq \alpha d_i / \lam$.

\paragraph{}

See Figure~\ref{fig:unobstructed-lower} for an illustration. Consider the isosceles triangle 
with vertices at $\q{t_i}$, $q_a$ and $q_b$, whose base $q_a q_b$ is perpendicular to the 
line $\p{t_i}\q{t_i}$.  The equal sides of the triangle have length $2d_i$, the base has 
length $2 \alpha d_i / \lam$, and let $q_c$ be the midpoint of the base. The target's strategy 
is to move from $\q{t_i}$ to either $q_a$ or $q_b$, and \emph{report its location $\qa{t}$
at the closest point on the line $\q{t_i} q_c$;} i.e. at all times, $\qa{t}$ is the perpendicular 
projection of $\q{t}$ onto the line $\q{t_i} q_c$. By the symmetric construction, and the
choice of the points $q_a$ and $q_b$, the tracker cannot tell whether the target is
moving to $q_a$ or $q_b$. Thus, any deterministic tracker makes an incorrect choice
in one of the two possible scenarios. For the \emph{worst-case performance bound}, 
we can equivalently assume that the target \emph{non-deterministically guesses} the 
tracker's intention, and moves to the better of the two possible locations, $q_a$ or $q_b$. 
The tracker makes this choice based on whether the tracker is on or below the line $\q{t_i} q_c$, 
or not. In the former case, the target moves to $q_a$, and to to $q_b$ otherwise. 
The $i$th phase terminates when the target reaches either $q_a$ or $q_b$, and the next phase
begins.
(We note that, after $i$ phases, there are $2^i$ possible choices made by the tracker, 
reflected in whether it is above or below the line $\q{t_i} q_c$ at the conclusion
of each phases. For each of these possible ``worlds'' there is a corresponding deterministic 
strategy of the target that ``fools'' the tracker in every phase, resulting in the maximum 
distance increase.)
There is one subtle point worth mentioning here. It is possible that during the phase, the distance 
between the players may shrink if the tracker temporarily moves towards the same final location as 
the target---however, our {\feasability} ensures that that the target's noisy location remains 
within the permissible error bound throughout the phase.
The following lemma shows that this simple strategy of the target can maintain the {\feasability} 
for any choice of $\alpha \leq 0.927$.

\begin{figure}
  \begin{center}
    \subfigure[] {
      \includegraphics[height=2cm]{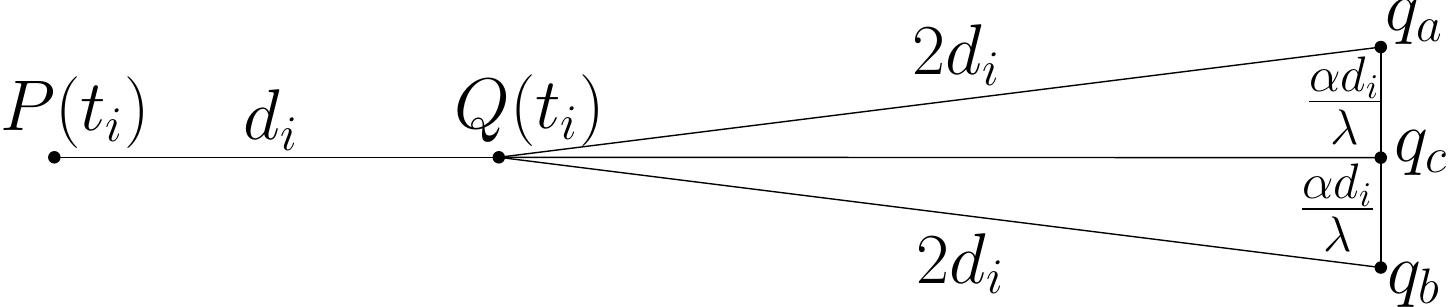}
      \label{fig:unobstructed-lower}
    }
    \subfigure[] {
      \includegraphics[height=2cm]{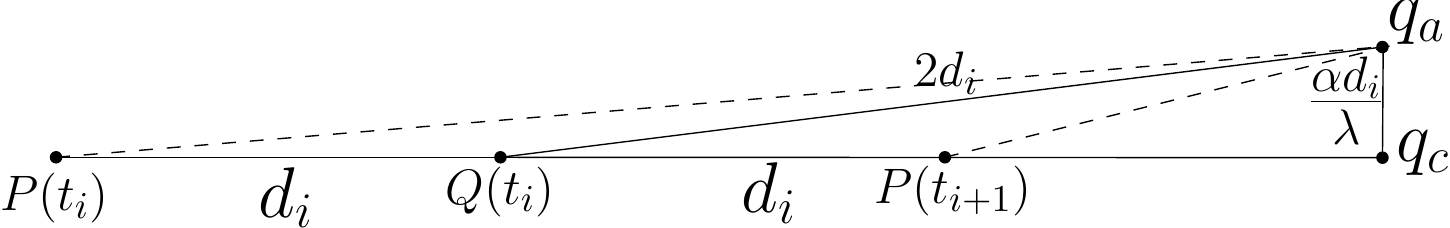}
      \label{fig:unobstructed-lower-geom}
    }
  \end{center}
  \caption{Target's strategy during the $i$th
    phase~\subref{fig:unobstructed-lower}, and proofs of
    Lemmas~\ref{lem:alphaNonSpeed} and~\ref{lem:lowerRoundChange}
    \subref{fig:unobstructed-lower-geom}.}
\end{figure}

\begin{lemma} \label{lem:alphaNonSpeed}
  The target can maintain the {\feasability} for any $\alpha \leq 0.927$.
\end{lemma}

\begin{proof}
  Consider an arbitrary phase $i$. By construction, we have $||\q{t}-\qa{t}||
  \leq \frac{\alpha d_i}{\lam}$ throughout this phase, so we only need to show 
   $\D{t} \geq \alpha d_i$. There is one subtle point worth mentioning here. 
   While the target's strategy will ensure that its distance from the tracker grows
	by a certain multiplicative factor \emph{at the end of the phase}, the distance between 
	the players may shrink during the phases. This happens when the tracker temporarily 
	moves towards the same final location as the target. In spite of this temporary ``lucky'' 
	guess by the tracker, we need to ensure that the target's noisy location remains within 
	the permissible error bound throughout the phase. The constant $\alpha$ is introduced 
	precisely to guarantee this validity, and we arrive at its value as follows.

  Let $d_{i+1}$ be the distance between $\p{}$ and $\q{}$ if both
  moved toward $q_a$ for the duration of phase $i$. Note that
  $d_{i+1}$ is the length of the segment $\p{t_i}q_a$ minus $2d_i$, as
  shown in Figure~\ref{fig:unobstructed-lower-geom}. The length of
  $\p{t_i}q_a$ can be calculated from the right triangle $q_a \p{t_i}
  q_c$, while the length of $q_a q_c$ is known by construction.
  Finally, $\q{t_i}q_c$ has length $d_i$ less than $\p{t_i}q_c$.
  Thus, we have:

  \begin{equation*}
    \begin{split}
      d_{i+1} & = \sqrt{(\sqrt{4d_i^2 - \frac{d_i^2\alpha^2}{\lam^2}} + d_i)^2 + \frac{d_i^2\alpha^2}{\lam^2}} - 2d_i\\
             & = \sqrt{5d_i^2 - \frac{d_i^2\alpha^2}{\lam^2} + 2d_i\sqrt{4d_i^2 - \frac{d_i^2\alpha^2}{\lam^2}} + \frac{d_i^2\alpha^2}{\lam^2}} - 2d_i\\
             %& = d_i\sqrt{5 + 2\sqrt{4 - \alpha^2/\lam^2}} - 2d_i\\
             & = d_i\sqrt{5 + 4\sqrt{1 - \frac{\alpha^2}{4\lam^2}}}- 2d_i
    \end{split}
  \end{equation*}

In order to satisfy the {\feasability}, we must choose an $\alpha$
such that the following inequality holds:

  \begin{equation*}
    \begin{split}
      \alpha d_i \:&\: \leq d_i\sqrt{5 + 4\sqrt{1 - \frac{\alpha^2}{4\lam^2}}} - 2d_i\\
       \alpha^2 + 4\alpha + 4 & \leq 5 + 4\sqrt{1 -\frac{\alpha^2}{4\lam^2}}\\
       \alpha^2 + 4\alpha + 4 & \leq 5 + 4 - \frac{\alpha^2}{2\lam^2}\\
        \alpha^2(1+\frac{1}{2\lam^2}) + 4\alpha - 5 & \leq 0
    \end{split}
  \end{equation*}

This gives the following upper bound:

 \begin{equation*}
    \alpha \leq \frac{-4 \pm \sqrt{16 +
        4(1+\frac{1}{2\lam^2})5}}{2(1+\frac{1}{2\lam^2})} \leq \frac{1}{3}(\sqrt{46}-4)
 \end{equation*}
%

%
%
%
%
% Mathematica function and constraints
% minimize (-4+sqrt{16+20*(1+x^2/2)})/(2*(1+x^2/2)), {0<=x<=1}
\end{proof}

%% By satisfying {\feasability} for the duration of a phase $i$, we know $||\q{t}-\p{t}|| \geq \alpha d_i$, 
%% and thus since $||\q{t}-\qa{t}||\leq \frac{\alpha d_i}{\lam}$, the position of $\qa{}$ falls within the 
%% noise constraints of the model. 
The preceding lemma shows that our construction satisfies the {\feasability}, and so we can
now lower bound the distance growth during a single phase.

\begin{lemma}
  At the start of phase $i+1$, we have $d_{i+1} \:\geq\: d_i\sqrt{1
    +\frac{\alpha^2}{2\lam^2}}$, where $\alpha = 0.927$ is an absolute
  constant.
  \label{lem:lowerRoundChange}
\end{lemma}

\begin{proof}
  Suppose, without loss of generality, that the target is at $q_a$ at
  the termination of the $i$th phase, which means the tracker is on or
  below the line $\q{t_i}q_c$. By the unit speed assumption, the
  target needs exactly $2d_i$ time for this move. The minimum value of
  $d_{i+1}$ is at least as large as if $\p{}$ had moved directly to
  $q_c$ by distance $2d_i$, as shown in
  Figure~\ref{fig:unobstructed-lower-geom}. We can calculate $d_{i+1}$
  from the right triangle $q_a \p{t_{i+1}} q_c$, as follows:

 \begin{equation*}
   \begin{split}
     d_{i+1} & \geq \sqrt{(\sqrt{4d_i^2 - \frac{\alpha^2d_i^2}{\lam^2}} -
              d_i)^2+\frac{\alpha^2d_i^2}{\lam^2}}\\
            & = \sqrt{d_i^2 - 2d_i\sqrt{4d_i^2 -
              \frac{\alpha^2d_i^2}{\lam^2}} +4d_i^2 -\frac{\alpha^2d_i^2}{\lam^2} +
              \frac{\alpha^2d_i^2}{\lam^2}}\\
            %& = d_i\sqrt{5 - 2\sqrt{4 - \frac{\alpha^2}{\lam^2}}}\\
            & = d_i\sqrt{5 - 4\sqrt{1 -\frac{\alpha^2}{4\lam^2}}}\\
            & \geq d_i\sqrt{5 - (4 - \frac{\alpha^2}{2\lam^2})}\\
            &  = d_i\sqrt{1 + \frac{\alpha^2}{2\lam^2}}
   \end{split}
 \end{equation*}

\end{proof}

We can now prove the main result of this section.

\begin{theorem}
  Under the relative error localization model, a target can increase its distance from an 
equally fast tracker at the rate of $\Omega(\lam^{-2})$.  In other words, the target can ensure that
  $\D{t} \:\geq\: \D{0} + \Omega(t /\lam^2)$ after any phase ending at time $t$.
  \label{thm:unobstructed-lower}
\end{theorem}
\begin{proof}
  The target follows the phase strategy, where that after the $i$th
  phase that lasts $2d_i$ time units, the distance between the tracker
  and the target is at least
  $d_i\sqrt{1+\frac{\alpha^2}{2\lam^2}}$. Therefore, the distance
  increases during the $i$th phase by at least the following
  multiplicative factor (using a Taylor series expansion):

\begin{equation*}
  \frac{d_i\sqrt{1+\frac{\alpha^2}{2\lam^2}} - d_i}{2d_i} \:\:=\:\:
  \frac{\sqrt{1+\frac{\alpha^2}{2\lam^2}} -1}{2} 
	\:\:\geq\:\:
  \frac{\alpha^2}{4\lam^2} - \frac{\alpha^4}{16\lam^4} 
	\:=\: \Omega(\frac{1}{\lam^2})
\end{equation*}

\cut{%---------
Finally, we apply a Taylor series expansion and obtain the desired
result.

\begin{equation*}
  \frac{\sqrt{1+\frac{\alpha^2}{2\lam^2}} -1}{2} = \frac{(1 + \frac{\alpha^2}{2\lam^2} -
    \frac{\alpha^4}{8\lam^4} + \cdots) -1}{2} \geq
  \frac{\alpha^2}{4\lam^2} - \frac{\alpha^4}{16\lam^4} = \Omega(\frac{1}{\lam^2})
\end{equation*}
}%-----------
\end{proof}

\section{Trackability with a Faster Tracker}
\label{sec:speed}

The results of the previous section establish bounds on the relative
advantage available to the target by the localization imprecision. Its
distance from the tracker can grow at the rate of $\Theta(\lam^{-2})$
with time. A tracking system can employ a number of different
strategies to compensate for this disadvantage.  In this section, we
explore one such natural mechanism: \emph{allow the tracker to move at
  a faster speed than the target.} A natural question then is: what is
the minimum speedup necessary to cancel out the localization noise as
a function of $\lam$? We give bounds on the necessary and sufficient
speedups, which match up to small constant factors as long as $\lam \geq
2$. The general form of the speedup function is $(1 -
\frac{1}{\lam^2})^{-1/2}$.  The following theorem proves the
sufficiency condition.

%% We first revisit our upper bound, and find an $\s$ sufficient for the
%% tracker to guarantee $\D{t} \leq \D{0}$.

\begin{theorem}
Suppose the target moves with speed one, and the tracker has speed $\s = 
\sqrt{\frac{1}{1-1/\lam^2}}$, where $\lam$ is the localization precision parameter.
Then, the tracker can maintain $\D{t} \leq \D{0}$, for all times $t \geq 0$.
  \label{thm:speed-upper}
\end{theorem}

\begin{proof}
  Our analysis closely follows the proof of
  Theorem~\ref{thm:unobstructed-upper}, and calculates the increase in
  the distance during time $\Delta t$. During this time, the tracker
  is able to move $\s \Delta t$, while the target can move at most
  $\Delta t$.  We can then calculate distance at time $t + \Delta t$
  from the triangle $abc$ (Fig.~\ref{fig:unobstructed-upper}), where
  $ab = \s \Delta t\sin\gamma$ and $ac = \D{t} + \Delta t - \s\Delta
  t\cos\gamma$, as follows:

  \begin{equation*}
    \begin{split}
     \D{t+\Delta t}& = \sqrt{(\s\Delta t\sin \gamma)^2 + (\D{t} + \Delta t - \s\Delta t\cos\gamma)^2}\\
                   & = \sqrt{\s^2\Delta t^2\sin(\alpha)^2 + \D{t}^2 + 2\D{t}\Delta t(1 -
                      \s\cos(\alpha)) + \Delta t^2 - 2 \Delta t^2 \s\cos(\alpha) + \Delta
                      t^2\s^2\cos(\alpha)^2}\\
                   & = \sqrt{\s^2\Delta t^2 + \D{t}^2 + 2\D{t}\Delta t(1 -
                       \s\cos(\alpha)) + \Delta t^2 - 2 \Delta t^2 \s\cos(\alpha)}\\
                   & = \sqrt{(\D{t} + \Delta t)^2 + \s^2\Delta t^2  -
                     2\D{t}\Delta t\s\cos(\alpha) - 2 \Delta t^2 \s\cos(\alpha)}\\
                   %& = \sqrt{\s^2\Delta t^2 + (\D{t} + \Delta t)^2
                   %  -2\Delta t(\D{t} + \Delta t) + 2\Delta t(\D{t}+
                   %  \Delta t) + \Delta t^2 
                   %  - \Delta t^2 - 2\D{t}\Delta t\s\cos(\alpha) - 2 \Delta
                   %    t^2 \s\cos(\alpha)  }\\
                   &  = \sqrt{(\D{t} + \Delta t - \Delta t)^2 + \s^2\Delta t^2 + 2\Delta t(\D{t} +
                       \Delta t) -\Delta t^2
                        - 2\D{t}\Delta t\s\cos(\alpha) - 2
                       \Delta t^2 \s\cos(\alpha) }\\
                   %&  = \sqrt{ \D{t}^2 + \s^2\Delta t^2 - \Delta t ^2 - 2\D{t}\Delta t\s\cos(\alpha) - 2
                   %    \Delta t^2 \s\cos(\alpha) + 2\Delta t(\D{t} + \Delta t) }\\
                   %& = \sqrt{ \D{t}^2 + \s^2\Delta t^2 - \Delta t ^2 + 2\Delta t(\D{t} + \Delta t)(1 - \s\cos(\alpha)) }\\
                   & = \D{t}\sqrt{ 1 +\s^2\Delta t^2/\D{t}^2 - \Delta t ^2/\D{t}^2 + 2\Delta
                       t(\D{t} + \Delta t)(1 - \s\cos(\alpha))/\D{t}^2 }\\
       & \leq \D{t} + \s^2\Delta t^2/2\D{t} - \Delta t ^2/2\D{t} + \Delta t(1 + \Delta t/\D{t})(1 - \s\cos\gamma)
    \end{split}
  \end{equation*}

\cut{%-------------
This leads to

  \begin{equation*}
    \begin{split}
      \Dt{t} &  \leq \lim_{\Delta t \rightarrow 0} \frac{\D{t+\Delta t} - \D{t}}{\Delta t}\\
             & = \lim_{\Delta t \rightarrow 0} \s^2\Delta t/2\D{t} - \Delta t/2\D{t} + (1 +
               \Delta t/\D{t})(1 - \s\cos \gamma)\\
             & = 1 - \s \cos \gamma
    \end{split}
  \end{equation*}

  Since $\sin \gamma \leq 1/\lam$, we have $\Dt{t} \leq 1 - \s
  \sqrt{1-1/\lam^2}$.  In order to ensure $\Dt{t} \leq 0$, it suffices
  that $\s \geq \sqrt{\frac{1}{1-1/\lam^2}}$.
}%----------

This allows us to bound $\Dt{t} \:\leq\: 1 - \s \cos \gamma$, from
which it follows that $\Dt{t} \leq 0$ as long as $\s \geq
\sqrt{\frac{1}{1-1/\lam^2}}$.
\end{proof}

We now show that if $\lam \geq 2$, this is the minimum speedup
necessary as a function of $\lam$, up to a small constant factor. We
use the phase-based strategy of Theorem~\ref{thm:unobstructed-lower},
however, the value of $\alpha$ determined by
Lemma~\ref{lem:alphaNonSpeed} is not sufficient to maintain the
{\feasability} in this case because of the higher speed of the
tracker.  Instead, the following lemma gives the sufficient choice of
$\alpha$.

\begin{lemma}
Let $\lam \geq 2$ and and $\alpha \leq 0.68$ be a constant.  Then, the {\feasability} can 
be maintained in any phase as long as $\s \:\leq\: \frac{1}{\sqrt{1-1/\lam^2}}$.
  \label{lem:alphaSpeed}
\end{lemma}
\begin{proof}
  As in Lemma~\ref{lem:alphaNonSpeed} we need only show that $\D{t}
  \geq \alpha d_i$, and again let $d_{i+1}$ be the distance between
  $\p{}$ and $\q{}$ if both moved toward $q_a$ for the duration of
  phase $i$. Note that $d_{i+1}$ is found as the length of the segment
  $\p{t_i}q_a$ as in Lemma~\ref{lem:alphaNonSpeed} except now minus
  $2\s d_i$ instead of $2d_i$. Thus, we have:

  \begin{equation*}
    \begin{split}
      d_{i+1} & = \sqrt{(\sqrt{4d_i^2 - \frac{d_i^2\alpha^2}{\lam^2}} +
           d_i)^2 + \frac{d_i^2\alpha^2}{\lam^2}} - 2\s d_i\\
              & = \sqrt{5d_i^2 - \frac{d_i^2\alpha^2}{\lam^2} +
                2d_i\sqrt{4d_i^2 - \frac{d_i^2\alpha^2}{\lam^2}} +
                \frac{d_i^2\alpha^2}{\lam^2}} - 2\s d_i\\
              %& = d_i\sqrt{5 + 2\sqrt{4 - \alpha^2/\lam^2}} - 2\s d_i\\
              & =  d_i\sqrt{5 + 4\sqrt{1 -
                  \frac{\alpha^2}{4\lam^2}}}- 2\s d_i
    \end{split}
  \end{equation*}

  %$$ d' = \sqrt{5d_i^2 - d_i^2\epsilon^2\alpha^2 + 2d_i\sqrt{4d_i^2 - d_i^2\epsilon^2\alpha^2} + d_i^2\epsilon^2\alpha^2} - 2Sd_i$$

  %$$ d' = d_i\sqrt{5 + 2\sqrt{4 - \epsilon^2\alpha^2}} - 2Sd_i$$

  In order to satisfy the {\feasability}, we must choose an $\alpha$
  such that the following inequality holds:

  \begin{equation*}
    \begin{split}
      \alpha d_i & \leq d_i\sqrt{5 + 4\sqrt{1 -
          \frac{\alpha^2}{4\lam^2}}} -2\s d_i\\
    \alpha^2 + 4\alpha\s + 4\s^2 & \leq 5 + 4\sqrt{1 -\frac{\alpha^2}{4\lam^2}}\\
    \alpha^2 + 4\alpha \s + 4\s^2 & \leq 5 + 4 - \frac{\alpha^2}{2\lam^2}\\
    \alpha^2(1+\epsilon^2/2) + 4\alpha \s + 4\s^2 -9 &\leq 0
    \end{split}
  \end{equation*}

  This gives the following upper bound when $\lam$ is minimum and $\s$
  is maximum, which by assumption is $2$ and $1/\sqrt{1-(1/2^2)}$,
  respectively.

 \begin{equation*}
   \alpha \leq \frac{-4\s + \sqrt{16\s^2 -
       4(1+1/2\lam^2)(4\s^2-9)}}{2(1+1/2\lam^2)} \leq 0.68
 \end{equation*}

\end{proof}

We can now prove a lower bound on the increase in the distance during
the $i$th phase.

\begin{lemma}
  If $\lam \geq 2$, $\alpha \leq 0.68$, and $\s \:\leq\:
  (1-1/\lam^2)^{-1/2}$, then at the start of the $i+1$ phase, we have
  $d_{i+1} \:\geq\: d_i\sqrt{(2\s-3)^2 + \alpha^2(\s - 1/2)/\lam^2}$,
  where $\alpha = 0.68$ is an absolute constant.
  \label{lem:lowerRoundChangeSpeed}
\end{lemma}
\begin{proof}
  Suppose, without loss of generality, that the target is at $q_a$ at
  the termination of the $i$th phase, which means the tracker is below
  the line $\q{t_i}q_c$. By the unit speed assumption, the target
  needs exactly $2d_i$ time for this move. The minimum value of
  $d_{i+1}$ is at least as large as if $\p{}$ had moved directly to
  $q_c$ by distance $2\s d_i$, as shown in
  Figure~\ref{fig:unobstructed-lower-geom}. We can calculate $d_{i+1}$
  from the right triangle $q_a \p{t_{i+1}} q_c$, as follows:

  \begin{equation*}
    \begin{split}
      d_{i+1} & \geq \sqrt{(\sqrt{4d_i^2 - \frac{\alpha^2d_i^2}{\lam^2}} -
        (2\s d_i-d_i))^2+\frac{\alpha^2d_i^2}{\lam^2}}\\
             &  =  d_i\sqrt{5 + 4\s^2 - 4\s - 2(2\s - 1)\sqrt{4 -\frac{\alpha^2}{\lam^2}}}\\
             & \geq  d_i\sqrt{5 + 4\s^2 - 4\s - 4(2\s - 1)(1 - \alpha^2/8\lam^2)}\\
             & =  d_i\sqrt{4\s^2 - 12\s + 9 + \s\alpha^2/\lam^2 - \alpha^2/2\lam^2}\\
             & = d_i\sqrt{(2\s -3)^2 + \alpha^2(\s - 1/2)/\lam^2}
    \end{split}
  \end{equation*}

\end{proof}

\paragraph{Remark.}
The preceding lemma can be used to calculate the maximum tracker speed
for which the target can still force a non-negative distance for a
specific $\lam$ as follows:
%%%  by solving Equation~\ref{eqn:speedDistanceGrowth} for $\s$ and obtaining Equation~\ref{eqn:speedLowerBound}.

\begin{equation*}
\begin{split}
\sqrt{(2\s -3)^2 + \alpha^2(\s - 1/2)/\lam^2} &= 1\\
4\s^2 - 12\s +\frac{\alpha^2\s}{\lam^2} & =
-8 + \frac{\alpha^2}{2\lam^2}\\
2\s - \left(\frac{12 - (\alpha/\lam)^2}{4}\right)^2 - \left(\frac{12 -
  \alpha^2/\lam^2}{4}\right)^2 & = -8 + \frac{\alpha^2}{2\lam^2}\\
2\s - \frac{12 - \alpha^2/\lam^2}{4} &= -\sqrt{-8 +
  \frac{\alpha^2}{2\lam^2}+ (\frac{12 - \alpha^2/\lam^2}{4})^2}\\
 \s & = \frac{-\sqrt{-8 + \frac{\alpha^2}{2\lam^2} +
    (\frac{12-\alpha^2/\lam^2}{4})^2} +
  \frac{12-\alpha^2/\lam^2}{4}}{2}
\end{split}
\end{equation*}

As $\lam$ gets large, the upper and lower bound are within a constant
factor of each other. Indeed, with a more careful choice of $\alpha$, 
we can show that the upper and lower bounds are within a factor of
$5.32$ (as opposed to $10.23$ for the above simple analysis) of each
other for $\lam \geq 2$, but we omit those details from this abstract.

\cut{%---------
\begin{equation}
   \sqrt{(2\s-3)^2 + \alpha^2(\s - 1/2)/\lam^2} = 1
\label{eqn:speedDistanceGrowth}
\end{equation}

\begin{equation}
\s = \frac{-\sqrt{-9 + \frac{\alpha^2}{2\lam^2} +
    (\frac{12-\alpha^2}{4\lam^2})^2} + \frac{12-\alpha^2}{4\lam^2}}{2}
\label{eqn:speedLowerBound}
\end{equation}

\begin{figure}[ht]
\begin{center}
\input{graphs/epsSpeedConstantDist.tex}
\caption{On the x-axis the error rate $\lam$ and on the y-axis the
  speed $\s$. The upper bound denotes the speed the tracker needs to
  guarantee $\D{t} \leq \D{0}$, the lower bound is the maximum speed
  the target can guarantee $\D{t} \geq \D{0}$ against, and the
  improved lower bound is the more detailed analysis of the same
  bound.}
\label{fig:epsSpeedConstantDist}
\end{center}
\end{figure}

We can improve the bound with a more careful choice of $\alpha$.
Lemma~\ref{lem:alphaSpeed} chooses an $\alpha$ that would
suffice for the maximum speed $\frac{1}{\sqrt{1-1/\lam^2}}$. However,
the $\s$ found by Equation~\ref{eqn:speedLowerBound} will not be this
large, and thus a larger $\alpha$ could have been chosen. Instead, if
we substitute $\alpha$ from Equation~\ref{eqn:alphaSpeed} into
Equation~\ref{eqn:speedDistanceGrowth} and solve for $\s$ we can
obtain a better bound. This is because we use the precise value of
$\alpha$ needed to compensate for the speed increase, rather than the
maximum $\s$ found by our upper bound.  While we cannot simplify the
resulting equation to a closed form for $\s$, using Mathematica we can
find the resulting value of $\s$ for any $\lam$, thus we include this
better bound in resulting in Figure~\ref{fig:epsSpeedConstantDist}. To
give some context, at an error rate of $5\%$ ($\lam = 20$), $\p{}$ can
maintain constant distance with speed a speed $1.001$, and $\q{}$ can
keep $\p{}$ from getting any closer that has a speed advantage up to
$1.0003$, a factor $3.33$ difference.
}%----------------

\section{Tracking in the Presence of Obstacles}
\label{sec:obstructed}

The presence of obstacles makes the tracking problem considerably harder under 
the localization noise. The following simple example (Fig.~\ref{fig:obstructed-lower})
shows that the target can grow its distance from the tracker as 
$\D{t} \:\geq\: \D{0} + t$, for any finite value of $\lam$. The obstacle 
consists of a single $U$-shaped non-convex polygon. Initially, the
target is at distance $\D{0}$ from the tracker, and the ``width'' of
the obstacle is less than $\D{0}/2 \lam$, so that the localization
error is unable to distinguish between a target moving inside the $U$
channel, or around its outer boundary. One can show that no matter
how the tracker pursues, its distance from the target can grow linearly
with time.

\medskip

\begin{figure}[ht]
\begin{center}
\includegraphics[height=2cm]{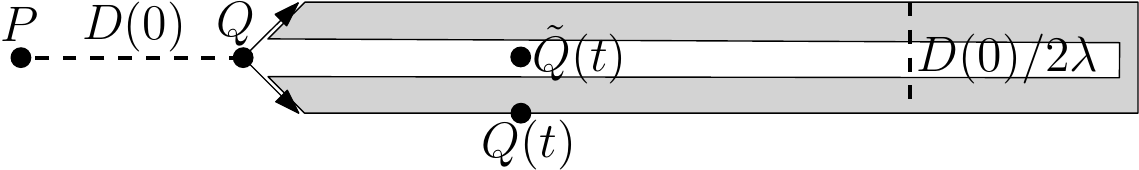}
\caption{Impossibility of tracking among obstacles.}
\label{fig:obstructed-lower}
\end{center}
\end{figure}

\paragraph{Path Proportionate Error.}
In order to get around this impossibility of tracking, we propose a
\emph{path proportionate error} measure, where the localization error
is proportional to the \emph{shortest path distance} between the
target and the tracker, and not the Euclidean distance as used
before. That is, the tracking signal and the physical movement of
the agents follow the same path metric.  Formally, the
localization error at time $t$ always obeys the following bound:

\begin{equation*}
\dist{\q{t}}{\qa{t}} \:\leq\: \frac{\dist{\p{t}}{\q{t}}}{\lam}
\end{equation*}

%% This model can be thought of as a sensor network, in which the signal must 
%% propagate along the shortest path in the sensor network to the tracker, and 
%% the noise is proportional to the distance the signal must propagate. 

We show that the best tracking performance in this model is $\D{t} \:=\: \D{0} + \Theta(t/\lam)$; 
that is the distance grows linearly with $1/\lam$, as opposed to the inverse quadratic function for the 
unobstructed case.

\subsection{Tracking Upper Bound}

The tracker's strategy in this case is also greedy, except now the tracker makes 
short-term commitments in \emph{phases}, instead of continuously changing its path 
towards the new observed location. In particular, for each phase, the tracker fixes 
its goal as the \emph{observed position of the target at the start of the phase},
moves along the shortest path to this goal, and then begins the next phase.

\begin{algorithm}[\discGreedy]
The initial phase begins at time $t=0$. During the $i$th phase, which begins at time $t_i$, 
the tracker moves along the shortest path to the observed location of the target at $t_i$, namely, 
$\qa{t_i}$. When tracker reaches $\qa{t_i}$, the $i$th phase ends, and the next phase begins.
\end{algorithm}

The upper bound on the tracking performance is given by the following theorem.

\begin{theorem}
Using {\discGreedy}, the tracker can ensure that $\D{t} \:\leq\: \D{0} + O(t/\lam)$.
\end{theorem}

\begin{proof}
  First note that because $\dist{\qa{t_i}}{\q{t_i}} \leq \D{t_i}/\lam$,
  it follows that $t_{i+1}-t_i = \D{t_i}+x\D{t_i}$, where
  $\frac{-1}{\lam} \leq x \leq \frac{1}{\lam}$. Thus, the target's
  progress during the $i$th phase is upper bounded as
  $\dist{\q{t_i}}{\q{t_{i+1}}} \leq \D{t_i} + x\D{t_i}$.
Next, by applying the triangle inequality, the distance between $\p{}$ and $\q{}$ at the beginning of phase $t_{i+1}$ is upper bounded as

  \begin{equation*}
    \begin{split}
      \dist{\p{t_{i+1}}}{\q{t_{i+1}}} & = \dist{\qa{t_i}}{\q{t_{i+1}}}\\
                                      & \leq \dist{\qa{t_i}}{\q{t_i}} +
                                        \dist{\q{t_i}}{\q{t_{i+1}}}\\ 
                                      & \leq \frac{\D{t_i}}{\lam} + \D{t_i} + x\D{t_i}
    \end{split}
  \end{equation*}

  Finally, the upper bound on the rate of distance increase can be
  derived as follows:

  \begin{equation*}
    \begin{split}
      \frac{\dist{\p{t_{i+1}}}{\q{t_{i+1}}} -
          \dist{\p{t_i}}{\q{t_i}}}{t_{i+1} - t_i} & \: \leq \:
        \frac{\D{t_i}  + \D{t_i}/\lam + x\D{t_i}
        - \D{t_i}}{\D{t_i} + x\D{t_i}} \\
      & = \frac{1/\lam + x}{1+x}  \:\leq\: \frac{2}{\lam+1}
    \end{split}
  \end{equation*}
  where the final inequality uses the fact that the minimum
  value occurs when $x=1/\lam$. In conclusion, during each phase the
  distance between the tracker and the target increases by at most a
  factor of $\frac{2}{\lam + 1}$, giving the bound $\D{t} \:\leq\:
  \D{0} + O(\frac{t}{\lam})$
\end{proof}

\subsection{Tracking Lower Bound.}

Our final result is to prove that the trackability achieved by
{\discGreedy} is essentially optimal. In particular, we construct an
environment with polygonal obstacles and a movement strategy for the
target that ensures $\D{t} \:\geq\: \D{0} + \Omega(t/\lam)$.  The
construction of the polygonal environment is somewhat complicated and
requires a carefully designed set of obstacles.  The main schema of
the construction is shown in Figure~\ref{fig:geodesicLowerGraph},
where each edge of the ``tree-like'' diagram corresponds to a
``channel'' bounded by obstacles, and each face corresponds to a
``gadget'' consisting of a group of carefully constructed obstacles,
with the outer face occupied entire by a single large obstacle.

\begin{figure}[htbp]
  \begin{center}
      \includegraphics[height=4cm]{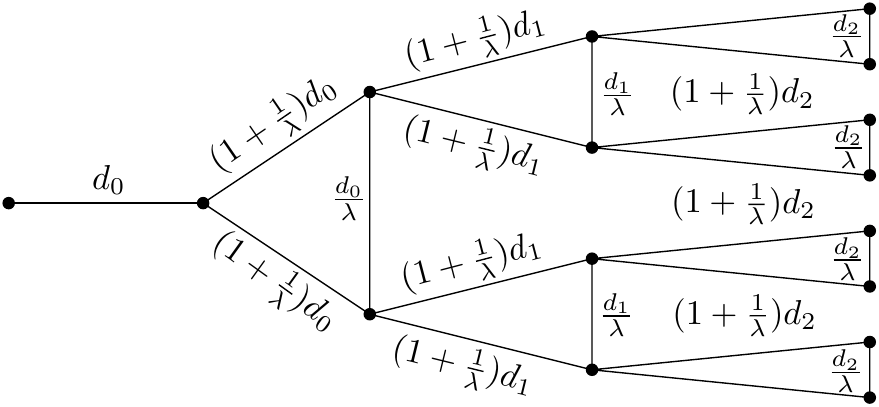}
  \end{center}
  \caption{A high level schema for the lower bound construction. The numbers
next to the edges denote the ``path length'' in the corresponding channels.}
  \label{fig:geodesicLowerGraph}
\end{figure}

As in the proof of Theorem~\ref{thm:unobstructed-lower}, the target
moves either to top or the bottom point of the gadget during a phase,
depending on the tracker's location.  The gadget construction is such
that the movement of the target along either path is indistinguishable
to the tracker because both paths are satisfied by a common set of
observed locations throughout the path. Thus, by invoking the earlier
equivalence principle, we may as well assume that the target knows the
tracker's choices.  If the target moves to the top, then the next
phase occurs in the top gadget, otherwise the bottom, and so on.

\cut{%------
We will show that the whether the
target takes the top or bottom path is indistinguishable to the
tracker. Therefore we assume the target can predict which path the
tracker takes, and thus choose whether to move to
the top or bottom accordingly. Again, this can be thought of as
demonstrating for any deterministic tracker strategy, there is a
corresponding target strategy in which $\p{}$ makes the wrong decision
in each phase, resulting in the maximum distance growth.
}%-----

To realize the geometric scheme of
Figure~\ref{fig:geodesicLowerGraph}, we replace each edge of the graph
with a channel as shown in Figure~\ref{fig:edgeConstruction}.  The
desired edge length can be realized by adding any number of
arbitrarily skinny bends such that the length of the shortest path
through each channel equals the edge length.  Each face is replaced
with a set of obstacles, called a gadget, see
Figure~\ref{fig:triangleConstruction} for an abstract illustration.
The jagged line between each pair of nodes corresponds to a channel
such that shortest path through that channel has the given length. The
target will move along the shortest path through either the top or
bottom channel while reporting its location in the center
channel. Meanwhile, the channels connecting the top and bottom to the
center will guarantee that $\dist{\q{t}}{\qa{t}}\leq
\frac{1}{\lam}\dist{\q{t}}{\p{t}}$ at all times $t$ during a phase.

\begin{figure}[ht]
  \begin{center}
    \subfigure[] {
      \includegraphics[height=3.6cm]{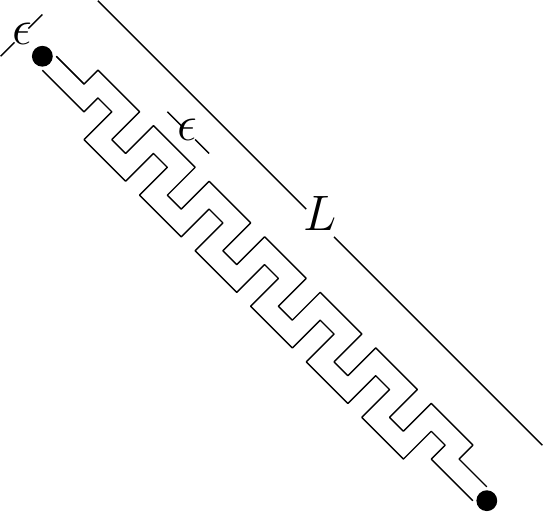}
      \label{fig:edgeConstruction}
    }
    \subfigure[] {
      \includegraphics[height=4cm]{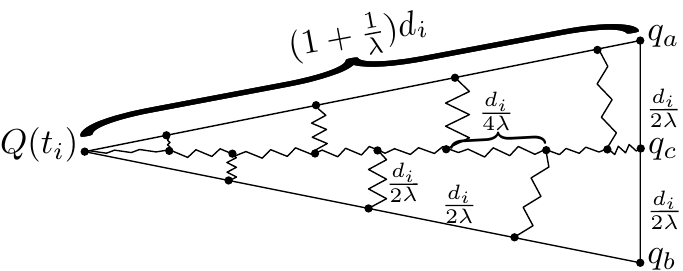}
      \label{fig:triangleConstruction}
    }
  \end{center}
  \caption{The channel construction in
    \subref{fig:edgeConstruction}. In
    \subref{fig:triangleConstruction} the shortest paths between nodes
    on the center path have length $\frac{d_i}{4\lam}$, and the
    remaining all have length $\frac{d_i}{2\lam}$.}
\end{figure}

\subsubsection{Gadget Construction and its Properties}

We now describe the construction of our gadgets and establish the
geometric properties needed for the correctness of our lower
bound. Each gadget is constructed out of two building blocks, the bent
channels seen in Figure~\ref{fig:edgeConstruction}, and intersections
depicted in Figure~\ref{fig:intersectionConstruction}. Each
intersection has the property that the shortest path between any two
of the points among $a,b$ and $c$ has length $2\delta$, where $\delta$
can be made arbitrarily close to $0$. Thus we can construct a channel
that branches into two channels such that the path length through the
intersection is the same regardless of the branch chosen. In
Figure~\ref{fig:gadgetConstruction}, we depict the construction of a
gadget using only intersections (triangles) and channels (jagged
lines).

\begin{figure}[ht]
  \begin{center}
    \subfigure[] {
      \includegraphics[height=3cm]{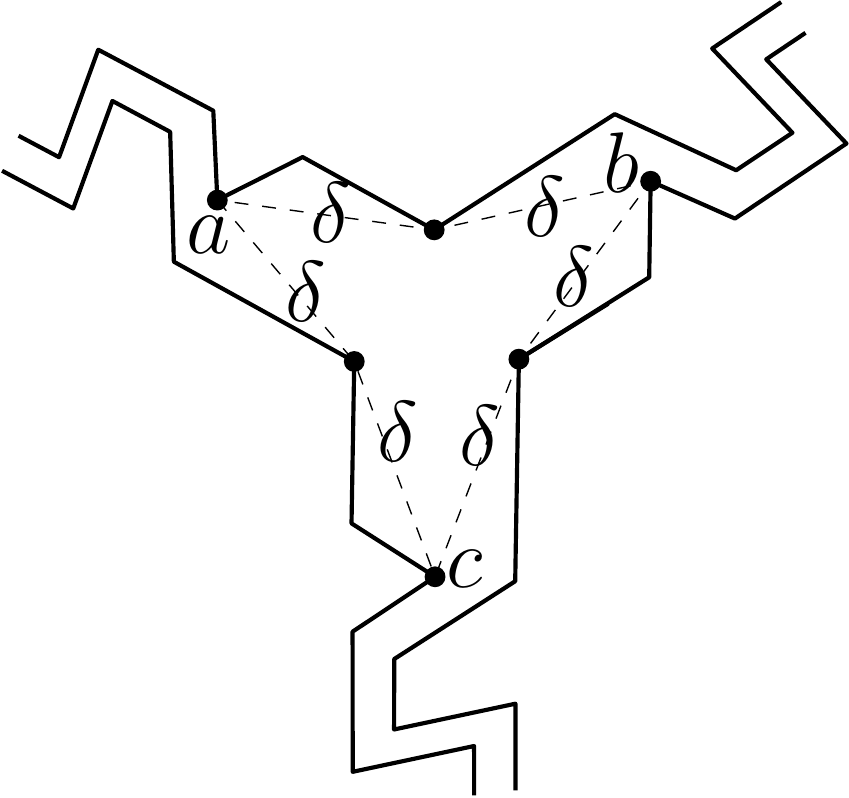}
      \label{fig:intersectionConstruction}
    }
    \subfigure[] {
      \includegraphics[height=3.5cm]{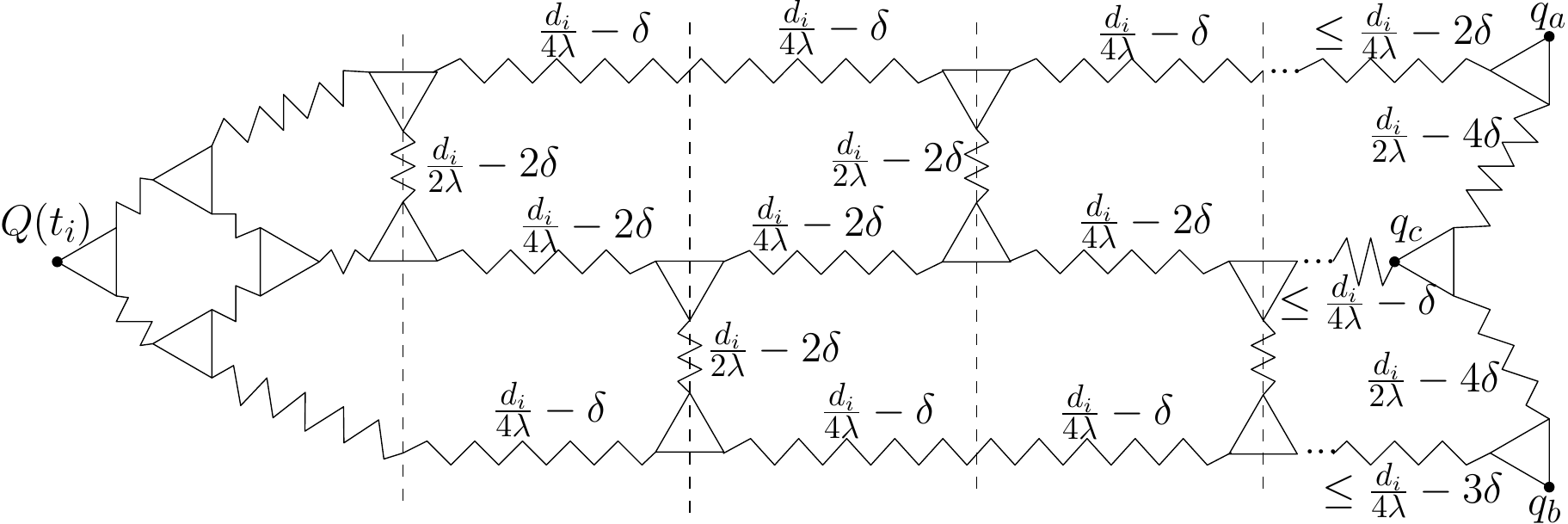}
      \label{fig:gadgetConstruction}
    }
  \end{center}
  \caption{In \subref{fig:intersectionConstruction} an example
    intersection such that the shortest path between any pair of $a$
    $b$ and $c$ has length $2\delta$. In
    \subref{fig:gadgetConstruction} an example gadget construction,
    where each triangle corresponds to an intersection with corners
    representing the points $a$ $b$ and $c$. The horizontal channels
    have length $\frac{d_i}{4\lam}$ between each pair of vertical
    dashed lines, except for the initial distance before the first
    line (which can be made arbitrarily small), and the remaining
    spillover distance after the last dashed line.}
\end{figure}

As in the lower bound for the unobstructed case, the target starts the phase at
$\q{t_i}$, and moves to $q_a$ or $q_b$ \emph{while the observed location of the
targets moves along the shortest path from $\q{t_i}$ to $q_c$}. In particular, let $\path{a}$,
$\path{c}$, and $\path{b}$ denote the shortest paths from $\q{t_i}$ to $q_a$, $q_c$ and 
$q_b$ respectively. The following lemma establishes several properties needed for the 
feasibility of the target's strategy.

\begin{lemma}
  We can construct a gadget for each phase $i$ such that 
(1) $\path{a}$, $\path{c}$ and $\path{b}$ have length
  $(1+\frac{1}{\lam})d_i$ and 
(2) for any point $x_c$ at distance $\ell$
  along $\path{c}$, the corresponding points $x_a$ and $x_b$ distance
  $\ell$ along $\path{a}$ and $\path{b}$, respectively, satisfy
  $\dist{x_c}{x_a} \leq \frac{d_i}{\lam}$ and $\dist{x_c}{x_b} \leq
  \frac{d_i}{\lam}$.
  \label{lem:gadgetProperties}
\end{lemma}

\begin{proof}
By construction, the shortest path in each channel between the dashed lines in
  Figure~\ref{fig:gadgetConstruction} has length $\frac{d_i}{4\lam}$, and
therefore this construction can be extended until $\path{a}$, $\path{c}$ and
  $\path{b}$ have length exactly $(1+\frac{1}{\lam})d_i$.
  Next, by the symmetry of the construction, we need only show that
  $\dist{x_c}{x_a} \leq d_i /\lam$. We ignore the case where $x_c$ lies in
  the channels before the first dashed lines, as the length of such
  channels can be made arbitrarily small to guarantee that
  $\dist{x_a}{x_c} \leq d_i/\lam$. The maximum distance between $x_a$
  and $x_c$ then occurs when $x_a$ lies at the midpoint between two
  intersections in the top channel. However, in this case one can
  easily verify that the following holds:
  \begin{equation*}
    \dist{x_c}{x_a} = \delta + \frac{d_i}{4\lam}-2\delta + 2\delta +
    \frac{d_i}{2\lam} - 2\delta + 2\delta + \frac{d_i}{4\lam}-\delta = \frac{d_i}{\lam}
\end{equation*}
This completes the proof.
\end{proof}

\subsubsection{Gap Invariant and the Proof of the Lower Bound}

We now formulate the invariant maintained by the target so that its motion is 
valid under our (path proportionate) localization error and achieves the 
desired lower bound.

\paragraph{\spFeasability .} 
Throughout the $i$th phase, the target moves along a path $\q{t}$ such that
$\D{t} \geq d_i$ for all times $t$, and all reported locations
satisfy $\dist{\q{t}}{\qa{t}} \leq \frac{d_i}{\lam}$.

\begin{lemma}
  For the duration of phase $i$, {\spFeasability} is maintained.
\end{lemma}

\begin{proof}
  Whether $\q{}$ moves along $\path{a}$ or $\path{b}$, they are both
  shortest paths (and this cannot be shortcut by $\p{}$), implying
  that $\D{t} \geq d_i$ for the duration of the phase. Without loss of
  generality, suppose $\q{}$ chooses $\path{a}$.  Then, after time
  $t$, both the target and its observed position have moved a distance
  of $t$ along $\path{a}$ and $\path{c}$, respectively. Therefore, by
  Lemma~\ref{lem:gadgetProperties}, we have $\dist{\q{t}}{\qa{t}} \leq
  \frac{d_i}{\lam}$.
\end{proof}

We can prove our lower bound.

\begin{theorem}
The target's strategy guarantees that after each phase ending at time $t$, the distance 
function satisfies $\D{t} \:\geq\: \D{0} + \Omega(\frac{t}{\lam})$.
\end{theorem}

\begin{proof}
The proof is by induction on the phase $i$. The basis of the induction is $i=0$.
Since the localization error makes target's top and bottom paths indistinguishable 
to the tracker, the target can ensure that at the end of phase $0$ the target is on
the side of $\path{c}$ that is opposite $\p{}$. Without loss of generality, suppose that
that target has reached $q_a$. Then the best case for $\p{}$ is if it moved 
$\frac{d_0}{\lam}$ along $\path{c}$, which achieves $\D{t_1} \geq \D{0} + \frac{\D{0}}{2\lam}$.

Now assume by induction that after phase $i-1$ ends at time $t_i$, we
have $\D{t_i} \geq \D{t_{i-1}} + \D{t_{i-1}}/2\lam = d_i$. Suppose now
that $\p{}$ has yet to reach the gadget corresponding to phase $i$
when $\q{}$ has finished phase $i$ at time $t_{i+1}$. Then necessarily
$\D{t_{i+1}} \geq d_i + d_i/\lam$, as that is the length $\path{a}$
and $\path{b}$. Otherwise if $\p{}$ has moved into the gadget, then
the inequality $\D{t_i} \geq d_i$ ensures that the closest the target
can be to the tracker is if $\p{}$ has moved $\frac{d_i}{\lam}$ along
$\path{c}$, which implies $\D{t_{i+1}} \geq \D{t_{i}} +
\frac{\D{t_i}}{2\lam}$.

Thus, in a round with duration $(1+\frac{1}{\lam})d_i$, the distance increases by at least 
$d_i /2\lam$. Thus, in the $i$th phase, the distance increases by a factor of at least

  \begin{equation*}
      \frac{d_i /2\lam}{(1+\frac{1}{\lam})d_i} =
         \frac{1}{2\lam(1+\frac{1}{\lam})}   = \frac{1}{2+2\lam}
  \end{equation*}
  
Thus, at the end of any phase, we have the inequality $\D{t} \:\geq\: \D{0} + \Omega(t/\lam)$,
which completes the lower bound.
\end{proof}

\section{Extension to $d$ dimensions}
\label{sec:dimension}

Our analysis of trackability was carried out for $2$-dimensional Euclidean
plane, but the results generalize easily to $d$ dimensions. Indeed, in the 
unobstructed case, our analysis of the upper bound only makes use of
the triangle inequality: the region of interest is the triangle formed by 
$\p{t}$, $\q{t}$, and $\qa{t}$, and the target $\q{}$ moves directly away 
from $\p{}$. Thus, within an arbitrarily small time interval $\Delta t$,
$\p{}$ and $\q{}$ are moving within the two-dimensional plane of
the triangle $\p{t}\q{t}\qa{t}$. The upper bound analysis therefore extend 
to any dimension $d \geq 2$. The same reasoning also holds in the presence of 
obstacles. Finally, the lower bound construction of $d=2$ immediately implies 
that the trackability lower bound holds in all dimensions $d \geq 2$.

\section{Conclusion}
\label{sec:conclusion}

Our paper is an attempt to formally study the worst-case impact of noisy 
localization on the performance of tracking systems. We analyzed a simple,
but fundamental, problem where a tracker wants to pursue a target but
the tracker's location sensor can measure the target's position only approximately,
with a relative error parameterized by quantity $\lam$. We prove upper and lower bound 
on the tracking performance as a function of this localization parameter $\lam$. 
A few surprising consequences of our results are (1) that the naive strategy of 
``always move to the observed location'' is asymptotically optimal, (2) the 
tracking error has a nice analytic form, showing an inverse quadratic dependence 
on $\lam$, and (3) under the natural ``path proportional error''for environments
with obstacles, the trackability has qualitatively a different dependence 
of the form $\Omega (1/\lam)$.

Compared to often-used heuristics such as Kalman filters, our work offers a 
more theoretical perspective for analyzing motion and localization errors in the 
presence of inevitable noise, which may be especially useful in situations 
where worst-case bounds are important, such as adversarial
tracking or surveillance. In addition to improving the constant factors
in our bounds, it will also be interesting to study the noisy sensing
model for other more complex settings.

\newpage

\bibliographystyle{plain}
\bibliography{main}

\end{document}